\documentclass{article}

\usepackage{times,natbib}
\usepackage{microtype}
\usepackage{graphicx}
\usepackage{booktabs}
\usepackage{hyperref}
\usepackage{xcolor}
\usepackage{amsmath}
\usepackage{amssymb}
\usepackage{amsthm}
\usepackage{enumitem}
\usepackage{float}

\definecolor{darkblue}{rgb}{0,0,0.5}
\hypersetup{colorlinks=true,citecolor=darkblue,linkcolor=darkblue,urlcolor=darkblue,hypertexnames=false}

\theoremstyle{plain}
\newtheorem{proposition}{Proposition}[section]

\theoremstyle{definition}
\newtheorem{definition}[proposition]{Definition}
\theoremstyle{remark}

\newcommand{\E}{\mathbb{E}}
\newcommand{\lengthSpamRaw}{0.919}
\newcommand{\lengthSpamMatched}{0.530}
\newcommand{\lengthJailbreakRaw}{0.836}
\newcommand{\lengthJailbreakMatched}{0.455}
\newcommand{\interventionToxicGain}{0.115}
\newcommand{\interventionHateGain}{0.101}
\newcommand{\interventionSpamLoss}{0.106}
\newcommand{\lotoDirectionCount}{4}
\newcommand{\lotoHolmMaxP}{0.010}
\newcommand{\independentSpearman}{0.89}
\newcommand{\crossModelAucConsistent}{19}
\newcommand{\crossModelCellCount}{20}
\newcommand{\crossModelSpearman}{0.89}
\newcommand{\civilAuc}{0.467}
\newcommand{\civilCiLow}{0.448}
\newcommand{\civilCiHigh}{0.487}
\newcommand{\civilTwoStageP}{0.0024}
\newcommand{\readoutHateVector}{0.644}
\newcommand{\readoutHateRadius}{0.444}
\newcommand{\readoutHateScalar}{0.547}
\newcommand{\readoutToxicVector}{0.656}
\newcommand{\readoutToxicRadius}{0.457}
\newcommand{\readoutToxicScalar}{0.537}
\newcommand{\readoutCivilVector}{0.555}
\newcommand{\readoutCivilLow}{0.497}
\newcommand{\readoutCivilHigh}{0.582}

\title{Useful Features, Backward Scores:\\OOD in Language-Model Trajectories}
\hypersetup{pdftitle={Useful Features, Backward Scores: OOD in Language-Model Trajectories},pdfauthor={Hamidreza Saghir}}
\author{Hamidreza Saghir\\Independent Researcher}
\date{}

\begin{document}
\maketitle

\begin{abstract}
Out-of-distribution (OOD) detectors prioritize inputs for closer inspection.  Yet features that distinguish input groups need not yield a useful anomaly ranking.  We analyze this gap in language-model trajectories under text-length control and fixed score directions.  On Spam development data, an input adaptation of D$^2$HScore falls from raw AUROC $\lengthSpamRaw$ to $\lengthSpamMatched$ after length matching.  On length-matched, held-out HateSpeech inputs, the same features yield AUROC $\readoutHateVector$ for a labeled linear classifier but $\readoutHateRadius$ for an ID-fitted distance score.  ToxicChat shows the same contrast.  Feature-selection and backbone controls retain the main reversal pattern.  Frozen Civil Comments and TweetEval irony tests also reverse ($\civilAuc$ and $0.435$), extending the finding beyond toxicity.  In these contrasts, anomalous groups have farther centers but tighter spread.  A labeled, fixed-center feature-space intervention changes rankings: equalizing spread helps some tasks and harms others.  OOD evaluation must check the chosen score's ranking even when its features distinguish the classes.
\end{abstract}

\section{Introduction}
\label{sec:introduction}

A distance-based toxicity detector uses non-toxic examples as its reference and ranks inputs for closer inspection by their distance from that reference.  If toxic inputs receive lower scores, they fall behind non-toxic ones in that ranking, even when the features distinguish the two groups.  This makes the score itself an evaluation target: informative features alone do not establish useful detection.  We study this gap for language-model trajectories, which summarize changes in internal representations across layers.  We test whether apparent detection survives length matching and whether the predefined distance score ranks the designated OOD inputs in the intended direction.

Length is the first problem.  On Spam development data, Input-D$^2$H gives raw AUROC $\lengthSpamRaw$, but only $\lengthSpamMatched$ after matching normal and anomalous inputs within shared token-length bins (Table~\ref{tab:raw_matched}).  This score adapts D$^2$HScore from generated responses to input tokens.  Other scores respond differently to matching, changing which methods appear useful.  Length control therefore changes the assessment of these input-domain scores.

After matching, a second problem remains.  On held-out HateSpeech inputs, a linear classifier trained on labeled examples gives matched AUROC $\readoutHateVector$, while distance from the in-distribution (ID) reference gives a below-chance $\readoutHateRadius$ using the same features (Figure~\ref{fig:radial_decomposition}A).  ToxicChat shows the same contrast.  These same-source results let us ask why features that distinguish the classes give the wrong distance ranking.

\begin{figure}[t]
\centering
\includegraphics[width=\columnwidth]{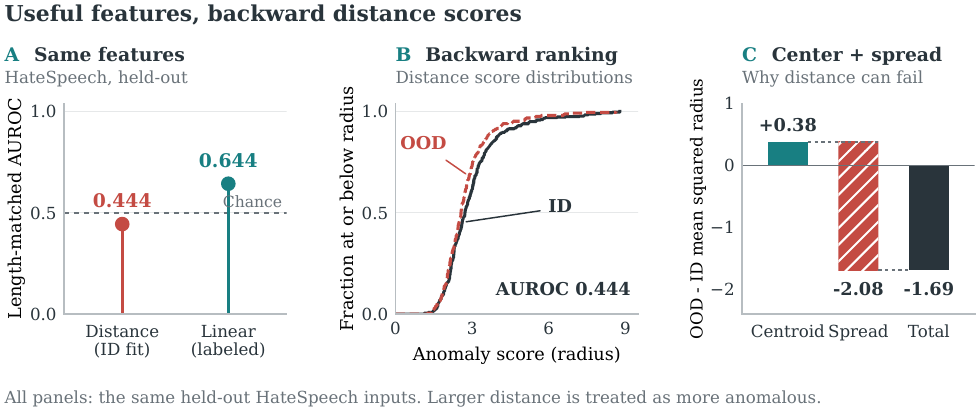}
\caption{\textbf{Useful features can produce a backward distance ranking.} All panels use the same length-matched, held-out HateSpeech inputs. \textbf{A:} A linear classifier trained on development labels gives above-chance AUROC; the ID-fitted distance score ranks the classes backward. These are descriptive same-source point estimates. \textbf{B:} OOD distances are smaller more often; a higher curve means more examples at or below that radius. \textbf{C:} Tighter OOD spread outweighs the farther center in mean squared distance. Curves and AUROCs average five matched subsets; repeated rows are not new observations. All task--split comparisons and distance-score uncertainty appear in Figure~\ref{fig:radial_detail} and Table~\ref{tab:score_uncertainty}.}
\label{fig:radial_decomposition}
\end{figure}

Why does the distance score fail?  On ToxicChat and HateSpeech, the out-of-distribution (OOD) center is farther from the ID reference, but the OOD group is tighter.  The detector scores individual points, not group centers: a tight group can have smaller distances for many of its members (Figure~\ref{fig:radial_decomposition}).  We test spread's role by expanding or shrinking OOD feature points at fixed centers.  Rankings change in both directions, and equalizing spread does not consistently repair them.

Our analysis of fixed trajectory scores does not propose a new detector.  Length matching, same-feature readouts and spread tests at fixed centers expose when features and rankings disagree.  New-data tests bound the result.  This extends backward-score analysis~\citep{nalisnick2019deep} to LM trajectories and makes score ordering a separate evaluation target.

\section{Representations and Scores}
\label{sec:geometric_views}

\label{sec:trajectory_radius}

Here OOD means different from the deployment reference, not absent from the model's pretraining data.  ToxicChat uses non-toxic conversations as ID and toxic conversations as OOD; HateSpeech, Spam and Jailbreak use their corresponding normal/anomalous labels.  The reference geometry is fit on ID examples only.  We test whether the score ranks the designated OOD class as more anomalous.

Each input becomes a point whose coordinates summarize its trajectory.  Four features describe changes in direction between adjacent layers (alignment and churn).  Two track how variation across token positions changes across layers.  Two compare the directions and relative sizes of attention and multilayer perceptron (MLP) outputs.  One measures alignment at a fixed early layer.  These nine features were chosen from 21 candidates using exploratory development labels; Appendix~\ref{app:features} gives every definition and selection decision.  Let $Z(x)$ denote the selected coordinates standardized using only ID references, with reference mean $\mu$ and sample covariance $\Sigma$.  The fixed trajectory radius is
\begin{equation}
  s_{\mathrm{traj}}(x)=\sqrt{q(Z(x))},\qquad
  q(z)=(z-\mu)^\top P(z-\mu),\quad P=(\Sigma+10^{-3}I)^{-1}.
  \label{eq:trajectory_radius}
\end{equation}
The detector treats larger radius as more anomalous.  AUROC is the probability that an OOD example receives a higher score than an ID example, counting ties as one half.  A value of $0.5$ is chance; below $0.5$, the ranking is reversed.  Any strictly increasing transformation of radius preserves AUROC: the problem is the ordering, not the numerical scale.  Choosing the score sign on labeled validation data makes the method semi-supervised.  Choosing it on test labels measures how well the classes can be separated, but does not evaluate the predefined detector (Appendix~\ref{app:radial_accounting}).

\section{Analysis Design}
\label{sec:protocol}

\subsection{Tasks and model}

The development suite contains ToxicChat~\citep{lin2023toxicchat}, Jailbreak~\citep{jackhhao2023jailbreak}, the binary Tweets Hate Speech Detection collection~\citep{tweetshatecard}, and SMS Spam~\citep{almeida2011sms}.  Qwen3-0.6B is the primary backbone; four additional backbones span 360M--4B parameters and four model families in total.  Appendix~\ref{app:data} gives model details, partitions and source limitations.

\subsection{Development and confirmation sets}

Each task has disjoint ID-reference, validation, development and confirmation partitions.  The 200--250 ID references alone fit standardization and precision.  Feature selection uses 200 development examples per class, greedily deleting from 21 candidates to maximize class separation while allowing either score sign.  The search uses development labels to choose a sign; evaluation keeps the sign fixed.  The complete path and tie rule appear in Appendix~\ref{app:features}.

The representation and extractor are frozen before confirmation analysis.  Trajectory analysis excludes inputs shorter than six tokens; the development comparison scores retain them.  All inputs are capped at 512 tokens.  Appendix~\ref{app:data} documents class counts and extraction checks.

\subsection{Length-matched evaluation}

Input length differs sharply between several classes.  We group inputs into shared token-length bins and sample equal ID and OOD counts within each bin containing both classes.  We compute AUROC over all retained examples, then average across the fixed matching seeds.  Labels determine which examples to retain, not their scores or sign.  Matching balances counts within bins, not exact lengths, and changes the evaluated population.  In Table~\ref{tab:raw_matched}, each method is matched on its own valid inputs, so subsets can differ across methods.  Appendix~\ref{app:data} gives the bin rule, seeds and retained counts; Appendix~\ref{app:robustness} tests sensitivity to bin count.

\subsection{Readouts and comparison scores}
\label{sec:baselines}

\paragraph{Same-feature readouts.}
One L2-regularized logistic classifier is trained jointly on length-matched ToxicChat and HateSpeech development inputs, using the nine ID-standardized coordinates after reference centering.  A scalar control instead applies logistic regression to a cubic-spline transformation of radius with eight ID-reference quantile knots and linear extrapolation.  Both use identical training rows, labels and weights, with equal weight for each source--class group and fixed regularization ($C=1$).  Source training is matched with seed 42.  We evaluate their decision scores on disjoint confirmation inputs, with no evaluation-set tuning.

\paragraph{Length-comparison scores.}
The length comparison adds three input-domain adaptations: CED-1C uses one-class oracle/auxiliary distances~\citep{lee2024ced}, Input-D$^2$H uses dispersion and attention-guided drift~\citep{ding2025d2hscore}, and Input-TV uses layerwise ID-distance volatility~\citep{wang2024trajectory}.  Appendix~\ref{app:adaptations} specifies their changed source settings, equations and fixed hyperparameters.  Each score keeps its predefined OOD direction.

\subsection{Feature-selection controls and prospective tests}

Four leave-one-task-out (LOTO) folds rerun the full 21-to-3 elimination path on three development tasks, then evaluate the excluded task with larger radius fixed as OOD.  Its labels do not influence subset search, although candidate design remains historical.  A retrospective cohort includes every cache meeting the recorded schema, ID-reference and class-count criteria: 11 evaluations, one sharing the development source.  Cross-model analysis reuses the four task contrasts.

Civil Comments~\citep{borkan2019nuanced} defines ID as toxicity $\le0.10$ and OOD as toxicity $\ge0.80$.  An initial frozen validation study was inconclusive.  A separate protocol then fixed the features, orientation, references and evaluation quotas before official-test feature extraction.  We report both stages without pooling; Appendix~\ref{app:selection_independent} records their full chronology and protocol hashes.

\section{Results}
\label{sec:results}

\subsection{Length differences inflate apparent performance}
\label{sec:length_results}

Does a strong raw score distinguish anomalous texts from normal texts of similar length?  On development data, Input-D$^2$H falls from $\lengthJailbreakRaw$ to $\lengthJailbreakMatched$ on Jailbreak and $\lengthSpamRaw$ to $\lengthSpamMatched$ on Spam after matching (Table~\ref{tab:raw_matched}).  Strong raw separation does not survive comparison within shared length ranges.  CED-1C drops less and retains signal on Spam; Input-TV improves on some tasks.  Length matching changes which scores appear useful.

\begin{table}[htbp]
\caption{\textbf{Length matching changes apparent detection performance.} Raw / length-matched AUROC on Qwen3-0.6B development data, computed separately for each method's valid inputs. Comparison scores retain inputs shorter than six tokens; trajectory radius excludes them. Matching changes the sample, not scores or their sign. Comparison methods are input-domain adaptations.}
\label{tab:raw_matched}
\centering
\small
\resizebox{\columnwidth}{!}{%
\begin{tabular}{lcccc}
\toprule
\textbf{Task} & \textbf{CED-1C} & \textbf{Input-D$^2$H} & \textbf{Input-TV} & \textbf{Trajectory radius} \\
\midrule
\csname @@input\endcsname tables/tab_raw_matched.tex 
\bottomrule
\end{tabular}%
}
\end{table}

Matching changes the evaluated population, not the texts.  The trajectory analysis retains only 30\% of valid Spam inputs, so its matched result describes shared lengths rather than the full source population.  This comparison diagnoses sensitivity to class-length differences and sample composition; it is not a causal text-length intervention.  For trajectory radius, matching preserves the positive Jailbreak and Spam rankings but leaves the ToxicChat and HateSpeech reversals.  Length control exposes one problem without resolving the other.

\subsection{Useful features, backward scores}
\label{sec:readout_results}

A linear classifier trained on ToxicChat and HateSpeech development labels gives matched AUROCs $\readoutToxicVector$ and $\readoutHateVector$ on held-out examples from those sources.  With the same nine features and evaluation rows, distance from the ID reference gives $\readoutToxicRadius$ and $\readoutHateRadius$.  Figure~\ref{fig:radial_decomposition} shows the HateSpeech example: the features support classification, but their distance score ranks the anomalous class backward.

A nonlinear readout of radius trained on the same labeled rows and weights reaches $\readoutToxicScalar$ and $\readoutHateScalar$ on the two confirmation tasks, below the vector readout's point estimates.

The comparison establishes useful within-source class information.  On the new Civil source, the vector classifier reaches $\readoutCivilVector$ $[\readoutCivilLow,\readoutCivilHigh]$; its paired gain over the scalar readout is $0.031$ $[-0.030,0.078]$, so transfer remains inconclusive.  Other distance scores also have task-dependent strengths (Appendix~\ref{app:ced_coverage}).  We now test whether the trajectory reversal depends on the original feature selection, backbone or evaluation examples.

\subsection{Does the reversal persist?}
\label{sec:radial_results}

\subsubsection{Does reversal survive feature and metric choices?}
\label{sec:selection_independent}

The reversal could depend on choosing features with development labels.  Rerunning subset selection while excluding each evaluation task preserves all \lotoDirectionCount/4 AUROC directions (Table~\ref{tab:selection_independent_main}).  ToxicChat and HateSpeech remain below chance; Jailbreak and Spam remain above it.  The excluded task's labels do not drive selection, although the candidate features were designed on the original suite.

All-candidate and fixed-sign checks preserve the point-estimate directions, and ToxicChat and HateSpeech remain below chance under every tested precision geometry.  Two fixed-sign intervals cross chance; Jailbreak confirmation can reverse under metric changes.  Appendices~\ref{app:selection_independent} and~\ref{app:robustness} give the full controls and uncertainty.

\begin{table}[t]
\caption{\textbf{Reversal after feature-selection control and on a new source.} LOTO selects subsets without the named task's labels. Civil Comments uses the separately frozen official-test protocol. Entries show fixed-orientation AUROC and conditional stratified-bootstrap 95\% intervals. Appendix~\ref{app:selection_independent} gives feature counts and the full center--spread decomposition.}
\label{tab:selection_independent_main}
\centering
\small
\begin{tabular}{lc}
\toprule
\textbf{Evaluation} & \textbf{AUROC [95\% CI]} \\
\midrule
\csname @@input\endcsname tables/tab_selection_independent_main.tex 
\bottomrule
\end{tabular}
\end{table}

\subsubsection{Does the pattern transfer across backbones?}

Five backbones spanning four model families keep AUROC on the same side of chance in \crossModelAucConsistent/\crossModelCellCount{} comparisons.  The exception is SmolLM2 ToxicChat, whose AUROC is at chance (Appendix Table~\ref{tab:cross_model_geometry}).  These checks reuse the same task contrasts and therefore test backbone sensitivity, not independent task replication.

\subsubsection{Does reversal persist on new data?}
\label{sec:confirmation}

\paragraph{Disjoint same-source confirmation.}
Held-out examples test whether the original samples drove the result.  All four AUROC point estimates retain their development directions, but the rankings are weaker.  The intervals cross chance for ToxicChat and Jailbreak, making these two confirmations inconclusive (Appendix Tables~\ref{tab:confirmation} and~\ref{tab:score_uncertainty}).

\paragraph{Separately frozen Civil Comments test.}
Civil Comments tests reversal on another source of toxicity data.  Matched AUROC is $\civilAuc$ $[\civilCiLow,\civilCiHigh]$, compared with $0.496$ raw.  The earlier validation stage was inconclusive; we froze the official-test protocol after that result, before extracting test features, and report both stages in Appendix~\ref{app:selection_independent}.  The test supports reversal among texts in length bins containing both classes.

The below-chance ranking persists across the recorded matching-bin choices and reference-refit checks; after matching, length alone has near-chance AUROC (Appendices~\ref{app:readiness} and~\ref{app:additional_robustness}).

\paragraph{Independent irony contrast.}
We froze a separate TweetEval irony test before extracting its official test inputs (Appendix~\ref{app:irony}).  With non-ironic tweets as ID and ironic tweets as OOD, matched radius AUROC is $0.435$ $[0.396,0.474]$ on 310 examples per class.  A within-source, label-trained linear readout on the same features reaches $0.614$ $[0.579,0.660]$; the prespecified ID-only 10-neighbor distance also reverses (Appendix~\ref{app:irony}).  The OOD center is farther but its spread is tighter.  This extends the observed geometry beyond toxicity without providing an ID-only correction.

\paragraph{Where reversal does not replicate.}
After excluding the overlapping JailbreakClassification cache, ten earlier cached evaluations give five intervals above chance and five inconclusive, with no further significant reversal in that cohort.  Appendix~\ref{app:selection_independent} retains the full cohort and overlap audit.  Reversal is a finding about particular task contrasts, not a prediction for every OOD class.

\subsection{How does spread affect the score?}
\label{sec:theory}

The reversals are not explained by a nearer OOD center.  All four OOD centers are farther from the reference than their ID counterparts on both development and confirmation splits.  ToxicChat and HateSpeech instead have smaller within-group spread (Figure~\ref{fig:radial_decomposition}C).  The score measures each point's distance from the reference, not its distance from its own group center.

A tight cluster and a wide cluster can have the same center but different distances from the reference.  To separate these effects, we write the mean squared distance as a center term plus a spread term.  For an evaluation group $G$, let $\nu_G$ be its center (centroid) and $C_G$ its covariance, computed with the group size as the divisor.  The standard mean-square identity gives:
\begin{equation}
 \E_G[q(Z)]
 =\underbrace{(\nu_G-\mu)^\top P(\nu_G-\mu)}_{\text{centroid displacement}}
 +\underbrace{\operatorname{tr}(P C_G)}_{\text{within-group dispersion}}.
 \label{eq:radial_decomposition}
\end{equation}
The first term measures the group center's distance from the reference; the second measures spread around that group center.  Subtracting ID from OOD gives $\Delta_{\mathrm{rad}}=\Delta_{\mathrm{cent}}+\Delta_{\mathrm{disp}}$.  A farther OOD center can coexist with smaller mean squared distances if the OOD points cluster tightly enough.  The reference geometry stays fixed, and no Gaussian class assumption is needed.  Appendix~\ref{app:radial_accounting} gives the derivation.

Across all eight task--split comparisons, the sign of the mean squared-radius difference agrees with whether AUROC is above or below chance (Figure~\ref{fig:radial_detail}).  This is an empirical finding, not a consequence of the identity: averages need not determine pairwise ranking.  Squaring the radius preserves AUROC, but mean radius and mean squared radius can order the groups differently.  We now test spread's effect on the individual rankings directly.

\subsubsection{Changing spread without moving the center}
\label{sec:dispersion_intervention}

We move every OOD feature point away from or toward its own center by one common factor, leaving every ID point in place.  Evaluation labels identify the groups.  This diagnostic covers the four confirmation tasks and both Civil stages.  We fixed the procedure after inspecting the original results but before computing its outcomes.  Within each matched subset, OOD features become
\begin{equation}
 Z'_{\mathrm{OOD}}=\nu_{\mathrm{OOD}}+\alpha(Z_{\mathrm{OOD}}-\nu_{\mathrm{OOD}}),
 \qquad \alpha_{\mathrm{eq}}=\sqrt{\frac{\operatorname{tr}(P C_{\mathrm{ID}})}{\operatorname{tr}(P C_{\mathrm{OOD}})}}.
 \label{eq:dispersion_intervention}
\end{equation}
Scales $\alpha>1$ spread the OOD points apart; scales $\alpha<1$ pull them together.  The OOD center stays fixed, and each point keeps its direction from that center.  The reference fit, matched rows and ID scores are unchanged.  The spread term scales by $\alpha^2$; at $\alpha_{\mathrm{eq}}$, its OOD-minus-ID contribution is zero.  This equalization scale depends only on class spread, not on which scale gives the best AUROC.  The experiment changes feature points, not text, and may produce features that no text would generate.  It cannot serve as a deployable correction or identify a linguistic cause.

\begin{figure}[t]
\centering
\includegraphics[width=\columnwidth]{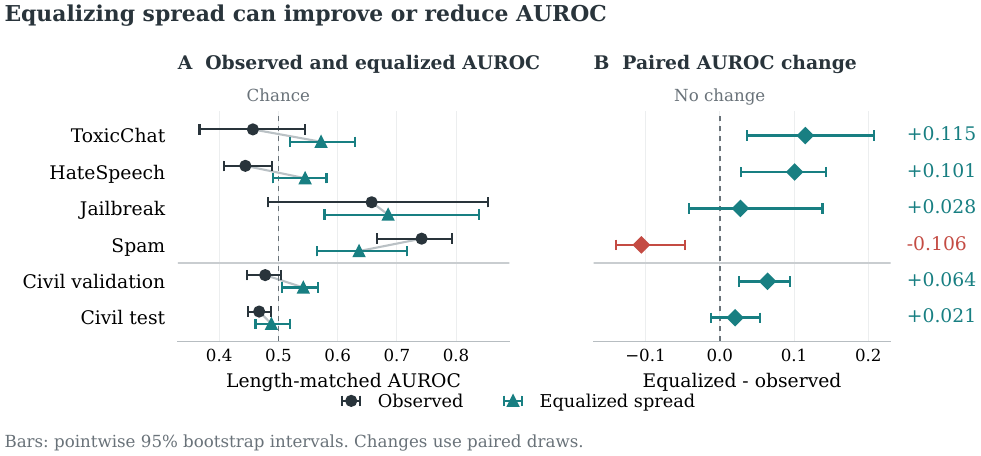}
\caption{\textbf{Equalizing spread can improve or reduce AUROC.} All six specified cases are shown: four same-source confirmation tasks above the divider, then the two Civil stages. \textbf{A:} Observed and equalized AUROC. \textbf{B:} Their paired difference. Intervals use 1,000 stratified paired bootstrap replicates, rematching and recomputing equalization in each replicate while holding the reference fit fixed. These are retrospective, pointwise intervals. Spam loses useful separation. Civil test has a positive mean squared-radius difference after equalization, but its AUROC point estimate remains below chance.}
\label{fig:dispersion_intervention}
\end{figure}

Equalization raises ToxicChat confirmation from $0.457$ to $0.572$ and HateSpeech from $0.444$ to $0.545$ (Figure~\ref{fig:dispersion_intervention}).  Paired changes are $+\interventionToxicGain$ $[0.036,0.208]$ and $+\interventionHateGain$ $[0.029,0.143]$; HateSpeech's equalized interval still crosses chance.  On Spam, reducing the larger OOD spread lowers AUROC from $0.742$ to $0.636$, a change of $-\interventionSpamLoss$ $[-0.140,-0.046]$.

Civil test separates mean-square correction from ranking recovery.  Equalization makes the OOD-minus-ID mean squared-radius difference positive ($+0.183$), but mean radius remains lower ($-0.0094$).  AUROC reaches $0.488$, with paired change $+0.021$ $[-0.011,0.054]$.  Correcting the mean squared distance does not establish an AUROC improvement: AUROC compares individual scores, not means.  Jailbreak's change is also inconclusive; Civil validation improves without changing its original verdict.  Appendix~\ref{app:dispersion_intervention} gives all values, scale curves and rank checks.

\section{Related Work}

Distance-based OOD detection scores deviations from an ID reference~\citep{lee2018simple,sun2022knn}; CED adds target--oracle and target--auxiliary relations~\citep{lee2024ced}.  \citet{nalisnick2019deep} report higher likelihood for OOD images and analyze location, variance and curvature.  Background statistics and complexity provide further explanations~\citep{ren2019likelihood,serra2020input}.  We examine ranking failure in LM trajectories, separating length sensitivity from reversal after matching.

DoSE estimates the ID density of model statistics~\citep{morningstar2021density}.  Mahalanobis++ studies feature norms and $\ell_2$ normalization~\citep{mueller2025mahalanobis}; \citet{janiak2026geometry} study spectral structure, local intrinsic dimensionality and ID-guided radial normalization.  Compact-ID objectives change representations during training~\citep{ruff2018deep}.  We fix the trajectory feature map and reference geometry, then test how class spread affects rankings.

TV-Score and D$^2$HScore use trajectories during generation~\citep{wang2024trajectory,ding2025d2hscore}; RAUQ and TrajGuard also require generation-time information~\citep{vazhentsev2025uncertainty,liu2026trajguard}.  Our Input-TV and Input-D$^2$H adapt this setting to input tokens (Appendix~\ref{app:adaptations}).

Length bias affects generation uncertainty~\citep{santilli2025revisiting}.  OOD shift type changes method comparisons on NLP tasks~\citep{arora2021types}, while NLP-ADBench finds no single method dominates across its eight datasets~\citep{li2025nlpadbench}.  We test length sensitivity in input-domain scores, then use the same-feature readout comparison to show why length control alone does not validate an anomaly ranking.

\section{Discussion and Limitations}

The evaluation lesson extends beyond this radius: an OOD claim concerns the ordering produced by a prespecified score, not just information in its features.  Our analysis exposes two distinct ways such a claim can be overstated.  Compare raw and length-matched performance to see whether apparent success survives among texts of similar length.  Then evaluate the score itself: useful class information and a farther OOD center do not guarantee the intended ranking.  The spread intervention shows why correcting the mean squared distance is insufficient.  Ranking depends on individual score comparisons.

ID-only scores remain useful on other task contrasts.  Alternative readouts help some cases without a general repair (Appendices~\ref{app:confirmation} and~\ref{app:readiness}).  Retrospective unit normalization of ID-centered features, with a new ID-only reference fit, raises Civil test matched AUROC from $0.467$ to $0.521$ $[0.497,0.545]$; the paired change is $+0.053$ $[+0.023,+0.084]$.  The normalized interval still includes chance, so this fixed control does not establish above-chance recovery.

\paragraph{Limitations.}
We diagnose score behavior without identifying why classes contract.  Linguistic composition, collection source and model behavior remain possible causes; jailbreak labels in particular differ in source composition.  Candidate-feature design is historical, the cache cohort is retrospective, and backbone checks reuse the original task contrasts.

Matching retains only shared length bins, leaving small samples for some tasks.  Intervals condition on feature design and assume independent rows rather than source clusters; intervention intervals also hold the reference fit fixed.  The feature-space intervention uses evaluation labels and may leave the set of realizable text features.  Results cover the tested trajectory probe in 360M--4B models, not frontier systems or all LM representations.

\paragraph{Conclusion.}
Length matching can remove apparent success while leaving backward rankings.  The same trajectory features can support within-source classification and produce below-chance distance rankings.  Fixed-center spread interventions connect this mismatch to the score's geometry.  Our analysis makes the scoring rule a separate evaluation target: test its intended ranking under length control instead of inferring detection from useful features alone.

\section*{Ethics Statement}
OOD detectors can affect moderation, abstention, and access decisions.  A reversing score can create false reassurance or disproportionate false positives.  We release selected numeric features and per-sample scores without text or source identifiers, distinguish supervised calibration from one-class detection, and recommend deployment only after task-representative validation.  Source datasets include toxic and hateful language; we release no new examples.

\section*{Reproducibility Statement}
The accompanying code supplement contains a verifier, fixed scoring utilities, and numeric caches for the length-matching contrast, Civil Comments and irony checks, spread intervention, readouts and endpoint comparison.  It recomputes their selected point estimates, regenerates the irony AUROC bootstrap, and checks included paired intervals from saved draws.  It does not rerun model extraction or every appendix analysis; full model-level replication needs the external datasets and checkpoints described in Appendix~\ref{app:data}.  Appendix~\ref{app:radial_accounting} gives the score identities and their assumptions.

\section*{AI-Use Statement}
GitHub Copilot assisted with implementation and testing, figure and table generation, manuscript and submission editing, literature and provenance checks, methodological critique, retrospective control design, mathematical argument review, and interpretation of analysis outputs.  The authors remain responsible for the scientific design, verification, and conclusions.

\bibliography{references}
\bibliographystyle{plainnat}

\clearpage
\appendix

\section{Guide to the Supporting Evidence}
\label{app:reproducibility}

The supporting material addresses five questions raised by the main results:
\begin{itemize}[leftmargin=*,itemsep=3pt]
  \item \textbf{Does the pattern replicate?} Appendix~\ref{app:selection_independent} gives disjoint confirmation, the full geometric comparison, cross-model results, leave-one-task-out selection and both Civil Comments stages. The cached-cohort analysis retains every eligible result and identifies the overlapping source.
  \item \textbf{What does changing spread establish?} Appendix~\ref{app:dispersion_intervention} reports every intervention outcome, the full scale response and paired uncertainty at fixed class centers and reference geometry.
  \item \textbf{Could selection or length explain the result?} Appendix~\ref{app:features} specifies every feature and selection step; Appendix~\ref{app:robustness} tests reference fitting, precision, feature subsets, data draws and matching bins.
  \item \textbf{Where do other scores work better?} Appendix~\ref{app:ced_coverage} compares endpoint, logit-based and trajectory scores, including the intent/topic settings where endpoint distance wins.
  \item \textbf{What methods and limits are documented?} Appendix~\ref{app:radial_accounting} proves the score identities, Appendix~\ref{app:adaptations} defines comparison scores, and Appendix~\ref{app:data} documents data provenance, compute and archive scope.
\end{itemize}

\section{Replication and Feature-Selection Controls}
\label{app:selection_independent}

Feature groups are A (layerwise alignment and churn), D (token spread), F (attention--MLP balance), and G (fixed-depth alignment).  ADFG-9 denotes the selected nine features; ADG-7 removes the two F features.

\subsection{Disjoint confirmation and the full geometric comparison}
\label{app:confirmation}

\begin{table}[H]
\caption{\textbf{Same-source held-out confirmation after feature development.}  ``ID-only'' fixes larger Mahalanobis distance as OOD.  ``Val. sign'' is selected on a disjoint approximately 50/50 validation split; the resulting score is semi-supervised.  $n$ reports valid examples before length matching.}
\label{tab:confirmation}
\centering
\small
\begin{tabular}{lcccc}
\toprule
\textbf{Task} & $n_{\mathrm{ID}}/n_{\mathrm{OOD}}$ & \textbf{ID-only} & \textbf{Val. sign} & \textbf{Val.-calibrated} \\
\midrule
\csname @@input\endcsname tables/tab_frozen_confirmation.tex 
\bottomrule
\end{tabular}
\end{table}

Choosing the radius direction using the separate validation labels raises four-task mean confirmation AUROC from $0.575$ to $0.625$.  All four selected signs transfer, but ToxicChat reaches only $0.543$.  This is semi-supervised calibration, not a change to the features.

Figure~\ref{fig:radial_detail} compares the centroid and dispersion contributions across both splits. Table~\ref{tab:radial_decomposition} gives the exact values underlying this visual comparison.

\begin{figure}[t]
\centering
\includegraphics[width=\columnwidth]{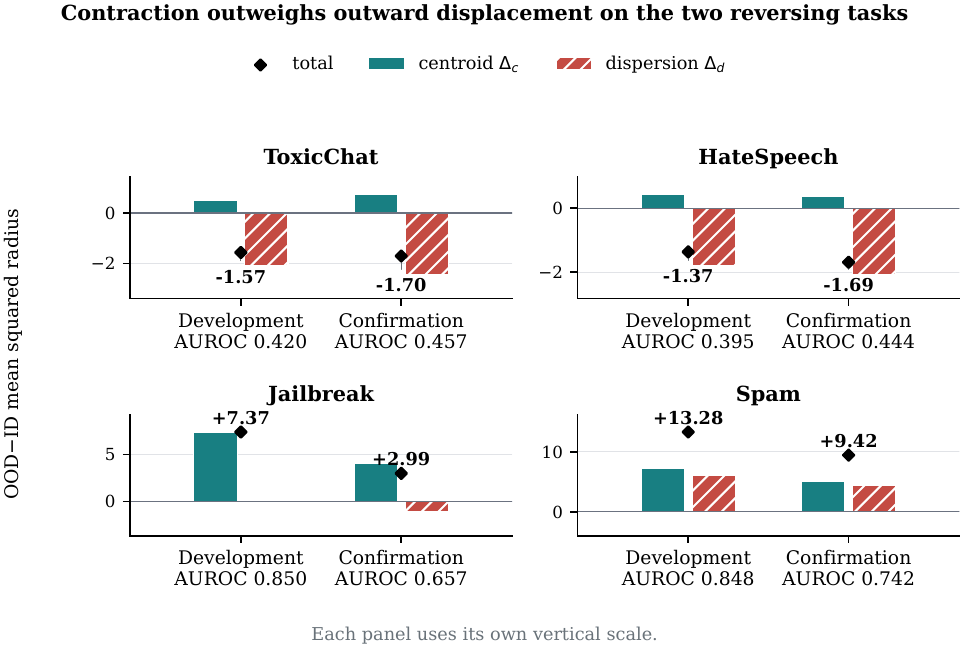}
\caption{\textbf{The full geometric comparison.} Solid teal bars show the OOD-minus-ID change in centroid contribution; hatched red bars show the dispersion contribution. Diamonds mark their sum, the change in mean squared radius. ToxicChat and HateSpeech contract enough to outweigh outward displacement on both splits. Each task has its own vertical scale. AUROCs below the split names are measured separately on the same length-matched examples, with larger radius fixed as OOD.}
\label{fig:radial_detail}
\end{figure}

\begin{table}[H]
\caption{\textbf{Exact decomposition of task-dependent radial direction.}  Five-seed length-matched means are shown for the development and disjoint confirmation splits.  Positive $\Delta_c$ means the OOD centroid is farther; negative $\Delta_d$ means the OOD cloud is tighter in the ID precision geometry.  $\Delta\E[q]=\Delta_c+\Delta_d$.}
\label{tab:radial_decomposition}
\centering
\small
\resizebox{\columnwidth}{!}{%
\begin{tabular}{lrrr@{\hspace{8pt}}rrr}
\toprule
& \multicolumn{3}{c}{\textbf{Development}} & \multicolumn{3}{c}{\textbf{Confirmation}} \\
\cmidrule(lr){2-4}\cmidrule(lr){5-7}
\textbf{Task} & $\Delta_c$ & $\Delta_d$ & $\Delta\E[q]$ & $\Delta_c$ & $\Delta_d$ & $\Delta\E[q]$ \\
\midrule
\csname @@input\endcsname tables/tab_radius_decomposition.tex 
\bottomrule
\end{tabular}%
}
\end{table}

\subsection{Cross-model geometry and attention--MLP features}

Table~\ref{tab:cross_model_geometry} reports AUROC and $\Delta\E[q]$ together.  The mean squared-radius difference matches the four-task sign pattern in all 20 cells.  AUROC stays on the same side of chance in 19, with SmolLM2 ToxicChat at $0.500$.  Across the 20 cells, $\Delta\E[q]$ and AUROC correlate at Spearman $\rho=\crossModelSpearman$.  The two attention--MLP features (Category F) improve Jailbreak on every backbone ($+0.069$ to $+0.242$) but have small or mixed effects elsewhere.  Their contribution depends on the task.

\begin{table}[H]
\caption{\textbf{Cross-model radial geometry.} Each task cell is fixed AUROC/$\Delta\E[q]$. $\Delta_F$ is the Jailbreak ADFG-9 minus ADG-7 AUROC.}
\label{tab:cross_model_geometry}
\centering
\scriptsize
\resizebox{\columnwidth}{!}{%
\begin{tabular}{lccccc}
\toprule
\textbf{Model} & \textbf{Toxic} & \textbf{Jailbreak} & \textbf{Hate} & \textbf{Spam} & $\Delta_F$ \\
\midrule
\csname @@input\endcsname tables/tab_cross_model_geometry.tex 
\bottomrule
\end{tabular}%
}
\end{table}

\subsection{Leave-one-task-out feature discovery}

Each fold repeats the full greedy path using only the other three tasks.  The excluded task supplies neither the selection objective nor its task-specific discovery sign.  Table~\ref{tab:loto_geometry} reports held-out AUROC with the score direction fixed.  Each bootstrap interval stays on the same side of chance as the corresponding development estimate, although ToxicChat's upper endpoint is $0.498$.  Global-label permutation tests give maximum Holm-adjusted $p=\lotoHolmMaxP$ but do not preserve the class--length association.  Within-bin tests, which preserve class counts in each length bin, give maximum adjusted $p=0.016$ (Appendix~\ref{app:readiness}).

\begin{table}[H]
\caption{\textbf{Leave-one-task-out geometry.}  $|F|$ is the feature count selected without the evaluation task.  AUROC uses larger radius as OOD and gives a pointwise stratified-bootstrap 95\% interval.}
\label{tab:loto_geometry}
\centering
\small
\begin{tabular}{lrcrrr}
\toprule
\textbf{Held out} & $|F|$ & \textbf{AUROC [95\% CI]} & $\Delta_c$ & $\Delta_d$ & $\Delta\E[q]$ \\
\midrule
\csname @@input\endcsname tables/tab_loto_geometry.tex 
\bottomrule
\end{tabular}
\end{table}

\paragraph{Fixed-orientation selection sensitivity.}
The LOTO search measures class separability, allowing either score sign during selection.  In a post hoc sensitivity check, we repeat the greedy path while maximizing fixed-direction AUROC on the other three tasks.  Table~\ref{tab:fixed_orientation_loto} shows that all four held-out AUROC estimates remain on the same side of chance and the mean squared-radius differences keep their signs.  ToxicChat and Jailbreak have intervals crossing chance.  Both selection checks reuse the same evaluation examples.

\begin{table}[H]
\caption{\textbf{Fixed-orientation selection sensitivity.}  The search objective $J_{\mathrm{fixed}}$ is mean fixed-orientation AUROC on the three training tasks.  Held-out AUROC retains larger radius as OOD and gives a pointwise 95\% interval.}
\label{tab:fixed_orientation_loto}
\centering
\small
\begin{tabular}{lrccr}
\toprule
\textbf{Held out} & $|F|$ & $J_{\mathrm{fixed}}$ & \textbf{Held-out AUROC [95\% CI]} & $\Delta\E[q]$ \\
\midrule
\csname @@input\endcsname tables/tab_fixed_orientation_loto.tex 
\bottomrule
\end{tabular}
\end{table}

\subsection{Additional cached evaluations}

After excluding JailbreakClassification, which overlaps the development source, five of ten eligible evaluations have intervals above chance and five are inconclusive.  None supplies another significant reversal.  The mean squared-radius sign agrees with AUROC direction in 9/10 cases, with Spearman $\rho=0.87$; WildGuard is the sole disagreement and retains only 14 matched examples per class.  Table~\ref{tab:independent_geometry} also shows the excluded row, giving six above-chance and five inconclusive intervals, 10/11 sign agreement and $\rho=\independentSpearman$ when included.  The underlying audit recorded ineligible caches, which are not part of the minimal code supplement.  Removing the known overlap does not establish independence among the remaining sources.

\begin{table}[H]
\caption{\textbf{Additional cached evaluations.} The asterisk marks a shared development source with exact feature overlap, excluded from the ten-evaluation summary. Ref. is the ID-only reference count; Eval. is ID/OOD before matching; $n_m$ is retained per class. Inference uses the conditional AUROC interval.}
\label{tab:independent_geometry}
\centering
\scriptsize
\resizebox{\columnwidth}{!}{%
\begin{tabular}{llrrccrrrl}
\toprule
\textbf{Dataset} & \textbf{Cohort} & \textbf{Ref.} & \textbf{Eval.} & $n_m$ & \textbf{AUROC [95\% CI]} & $\Delta_c$ & $\Delta_d$ & $\Delta\E[q]$ & \textbf{Inference} \\
\midrule
\csname @@input\endcsname tables/tab_independent_geometry.tex 
\bottomrule
\end{tabular}%
}
\end{table}

\begin{figure}[htbp]
\centering
\includegraphics[width=\columnwidth]{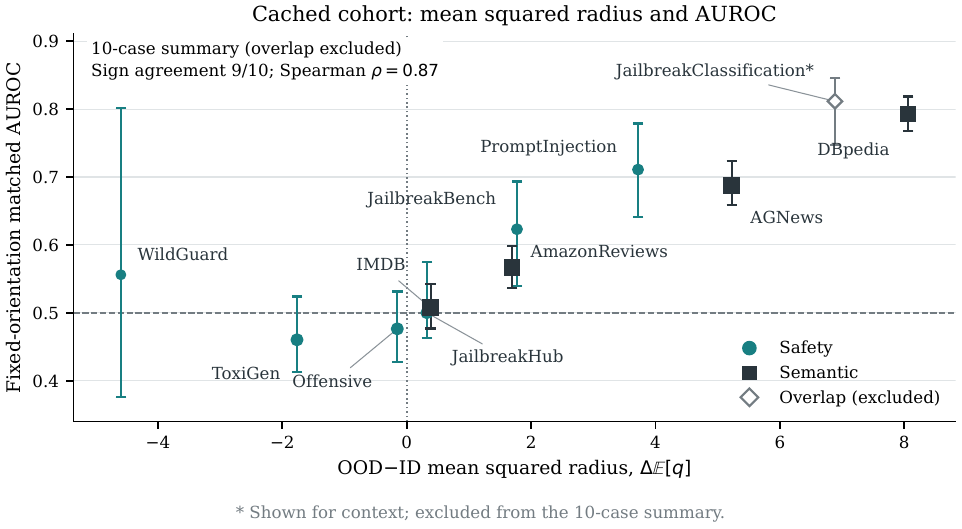}
\caption{\textbf{Mean squared-radius differences and AUROC across the cached cohort.} The headline summary excludes JailbreakClassification, shown as a hollow diamond and marked with an asterisk because it overlaps development. Circles and squares distinguish safety and semantic tasks; marker size reflects matched sample count. Error bars are stratified-bootstrap 95\% intervals. WildGuard is the sole sign disagreement and has the smallest sample. Table~\ref{tab:independent_geometry} includes the overlapping row.}
\label{fig:independent_geometry}
\end{figure}

\subsection{Civil Comments validation and test}

Civil validation was inconclusive; the separately evaluated official test ranks OOD below ID (Table~\ref{tab:prospective_civilcomments}).  The validation candidate pool initially targeted 5,000 OOD examples, but only 1,633 met toxicity $\ge0.80$.  Before model loading, we allowed every eligible OOD example into the candidate pool while keeping thresholds, quotas, seed, features, orientation and decision criteria fixed.  After observing validation, we doubled the official-test evaluation quota based on class-count feasibility, before extracting test features; all other choices remained fixed.  We report both stages without pooling.  The global-label two-stage adjusted $p=\civilTwoStageP$ does not condition on length; a within-bin permutation check gives $p=0.004$.  Both A/D/G-only controls retain the negative direction.

\begin{table}[H]
\caption{\textbf{Two-stage prospective Civil Comments result.}  Validation uses 250 ID references and 500 evaluation examples per class; official test uses 500 and 1,000.  Only ADFG-9 rows carry bootstrap intervals because they define the prospectively specified primary criterion.}
\label{tab:prospective_civilcomments}
\centering
\small
\begin{tabular}{llcrrrl}
\toprule
\textbf{Stage} & \textbf{Features} & \textbf{AUROC [95\% CI]} & $\Delta_c$ & $\Delta_d$ & $\Delta\E[q]$ & \textbf{Verdict} \\
\midrule
\csname @@input\endcsname tables/tab_prospective_civilcomments.tex 
\bottomrule
\end{tabular}
\end{table}

\subsection{Independent TweetEval irony check}
\label{app:irony}

On held-out irony tweets, the fixed distance score reverses even though the same features support a within-source labeled readout.  TweetEval's binary irony task~\citep{barbieri2020tweeteval,vanhee2018irony} defines non-irony as ID and irony as OOD.  Source revision \texttt{b3a375baf0f409c77e6bc7aa35102b7b3534f8be} supplies 2,862 official train, 955 validation and 784 test rows.  Before seeing target scores, we fixed the nine features, Qwen3-0.6B checkpoint, positive radius direction, one ID-only alternative and all decision thresholds.  After the six-token filter and within-source duplicate checks, we drew 250 train ID references and disjoint 300/class train examples for a labeled logistic diagnostic.  We did not use the validation split.  The test contains 472 ID and 310 OOD examples; all five length-matching seeds retain 310/class.  A checksum-verified overlap audit completed after scoring found zero matches against 6,149 texts from the other datasets evaluated here; the sampled rows did not change.  Eight feature and token-length parity checks passed before target extraction.

The ID-only alternative is mean distance to the ten nearest references in the ID-standardized nine-dimensional space; larger scores denote OOD for both scores.  Raw radius, 10-neighbor and linear-readout AUROCs are $0.412$, $0.397$ and $0.610$; matched values are $0.435$ $[0.396,0.474]$, $0.415$ $[0.379,0.459]$ and $0.614$ $[0.579,0.660]$.  Matched length-only AUROC is $0.500$.  The neighbor-minus-radius difference is $-0.020$ $[-0.039,0.005]$.  All intervals use 1,000 within-class paired test-row draws, rematching each draw with references and readout fixed.  The prespecified ID-only success criterion required neighbor AUROC at least $0.65$ and a gain of at least $0.10$, among other conditions; neither target was reached.  No alternate score or source was tried in response.

Applying Equation~\ref{eq:radial_decomposition} after this result gives matched differences of $+0.273$ $[0.156,0.412]$ for center displacement, $-1.748$ $[-2.458,-0.993]$ for spread and $-1.475$ $[-2.149,-0.725]$ for mean squared radius, with the same fixed-reference row bootstrap.  This post hoc geometry check does not identify a linguistic cause or change the failed ID-only criterion.  Intervals here are pointwise and do not account for sequential task choice, repeated authors or reference uncertainty.  The irony subset has no stated license and is subject to Twitter terms; no tweet text is redistributed.

\section{Score Orientation and Mean Squared Radius}
\label{app:radial_accounting}

\begin{proposition}[Centroid--dispersion decomposition]
\label{prop:radial}
For a standardized random vector $Z$, a fixed reference mean $\mu$ and $P\succeq0$, Equation~\ref{eq:radial_decomposition} holds for any group with finite second moments.  Its empirical form computes evaluation covariance with the group size as the divisor, regardless of the reference covariance estimator.
\end{proposition}
\begin{proof}
Write $Z-\mu=(Z-\nu_G)+(\nu_G-\mu)$, expand the quadratic and take expectations.  The cross term vanishes because $\E_G[Z-\nu_G]=0$; the centered quadratic has expectation $\operatorname{tr}(P C_G)$.
\end{proof}

\begin{definition}[Fixed-orientation detection and separability]
For independent $X_{\mathrm{ID}}$ and $X_{\mathrm{OOD}}$, define
\begin{equation}
 A(s)=\Pr[s(X_{\mathrm{OOD}})>s(X_{\mathrm{ID}})]
 +\tfrac12\Pr[s(X_{\mathrm{OOD}})=s(X_{\mathrm{ID}})].
\end{equation}
Fixed-orientation detection evaluates the predefined $s$ using $A(s)$.  Allowing either score sign measures class separation as $S(s)=\max\{A(s),A(-s)\}$.
\end{definition}

\begin{proposition}[Orientation is part of the detector]
\label{prop:orientation}
For every scalar score, $A(-s)=1-A(s)$ and $S(s)=\tfrac12+|A(s)-\tfrac12|$.  Choosing between $s$ and $-s$ with evaluation labels measures separability.  It cannot evaluate the predefined detector.
\end{proposition}
\begin{proof}
Negation exchanges greater-than and less-than events while preserving ties.  Their probabilities sum to one, with half of each tie assigned to either AUROC.
\end{proof}

\section{Exact Trajectory Features and Selection Provenance}
\label{app:features}

For each adjacent layer pair, the probe averages tokenwise cosine over positions $1,\ldots,T-1$, excluding the final input position:
\begin{equation}
 c_l(x)=\frac{1}{T-1}\sum_{t=1}^{T-1}\cos\!\left(h_l^{(t)},h_{l+1}^{(t)}\right).
 \label{eq:layer_cosine}
\end{equation}
Let $a_l^{(t)}$ and $m_l^{(t)}$ denote the attention and MLP outputs at layer $l$ and position $t$.  The trajectory radius uses:
\begin{align}
  \texttt{hs\_cos\_mean} &= \operatorname{mean}_l c_l, &
  \texttt{hs\_cos\_min} &= \min_l c_l,\\
  \texttt{hs\_churn\_late} &= \operatorname{mean}_{l\in\text{last third}}(1-c_l),\\
  \texttt{hs\_churn\_ratio} &=
  \frac{\operatorname{mean}_{l\in\text{last third}}(1-c_l)}
       {\operatorname{mean}_{l\in\text{first third}}(1-c_l)+10^{-10}}.
\end{align}
For representation spread, define
\begin{equation}
  q_l=\frac{1}{d}\sum_{j=1}^{d}
  \operatorname{Std}_{t=1,\ldots,T-1}\!\left(h_{l,j}^{(t)}\right).
\end{equation}
The tokenwise standard deviation uses sample normalization (denominator $T-2$ for $T-1$ positions).  Then \texttt{repr\_spread\_slope} is the least-squares slope of $q_l$ against layer index, and \texttt{repr\_spread\_std} is its population standard deviation across layers.

For components, define layerwise averages
\begin{align}
 A_l&=\operatorname{mean}_{t<T}\|a_l^{(t)}\|_2,&
 M_l&=\operatorname{mean}_{t<T}\|m_l^{(t)}\|_2,\\
 R_l&=\frac{A_l}{A_l+M_l+10^{-10}},&
 U_l&=\operatorname{mean}_{t<T}\cos(a_l^{(t)},m_l^{(t)}).
\end{align}
The retained features are \texttt{comp\_align\_mean}$=\operatorname{mean}_l U_l$ and \texttt{comp\_attn\_ratio\_shift}, the last-third mean of $R_l$ minus its first-third mean.  Finally, \texttt{hs\_cos\_at\_p17} is $c_{\lfloor L/6\rfloor}$.

These definitions are used for every reported trajectory-radius result.  Category A is a token-averaged inter-layer cosine, and churn is defined as $1-c_l$.

\subsection{Exact exploratory selection procedure}

\paragraph{Candidate pool and order.}
The search used the following 21 coordinates; the displayed order is also the deterministic tie-breaking order:
\begin{itemize}[leftmargin=*,itemsep=1pt]
  \item A: \texttt{hs\_cos\_mean}, \texttt{hs\_cos\_std}, \texttt{hs\_cos\_min}, \texttt{hs\_cos\_max}, \texttt{hs\_cos\_range}, \texttt{hs\_churn\_early}, \texttt{hs\_churn\_late}, \texttt{hs\_churn\_ratio};
  \item D: \texttt{repr\_spread\_slope}, \texttt{repr\_spread\_std}, \texttt{repr\_spread\_ratio};
  \item F: \texttt{comp\_attn\_ratio\_mean}, \texttt{comp\_attn\_ratio\_std}, \texttt{comp\_align\_mean}, \texttt{comp\_align\_std}, \texttt{comp\_attn\_ratio\_shift};
  \item G: \texttt{hs\_cos\_at\_p17}, \texttt{hs\_cos\_at\_p33}, \texttt{hs\_cos\_at\_p50}, \texttt{hs\_cos\_at\_p67}, \texttt{hs\_cos\_at\_p83}.
\end{itemize}
For completeness, these names exhaust the following deterministic operators.  A contains mean, population standard deviation, minimum, maximum, and range of $(c_l)_l$, followed by first-third and last-third means of $1-c_l$ and their ratio.  D contains the least-squares slope and population standard deviation of $(q_l)_l$, and $q_{\mathrm{final}}/(q_{\mathrm{first}}+10^{-10})$.  F contains the mean and population standard deviation of $(R_l)_l$, the mean and population standard deviation of $(U_l)_l$, and the last-third minus first-third mean of $R_l$.  G samples $c_l$ at clipped indices $\lfloor\rho L\rfloor$ for $\rho\in\{1/6,1/3,1/2,2/3,5/6\}$.  Every third has size $\max(1,\lfloor L/3\rfloor)$.

\paragraph{Objective.}
For candidate subset $F$, task $t$, and its positive Mahalanobis radius $r_{t,F}$, development labels first determine
\begin{equation}
 \sigma_{t,F}=\begin{cases}
 +1,&\operatorname{AUC}_{\mathrm{raw}}(r_{t,F})\ge\tfrac12,\\
 -1,&\operatorname{AUC}_{\mathrm{raw}}(r_{t,F})<\tfrac12.
 \end{cases}
\end{equation}
The search objective is
\begin{equation}
 J(F)=\frac{1}{4}\sum_{t=1}^{4}\frac{1}{5}
 \sum_{b\in\{42,123,456,789,1337\}}
 \operatorname{AUC}_{\mathrm{LM},b}(\sigma_{t,F}r_{t,F}),
 \label{eq:selection_objective}
\end{equation}
where $\operatorname{AUC}_{\mathrm{LM},b}$ is ten-bin length-matched AUROC for matching seed $b$.  Thus $J$ measures development-set class separation after choosing the more favorable score sign.  It is not an ID-only detector estimate.

\paragraph{Elimination and stopping.}
Starting from all 21 coordinates, at each iteration we evaluate all one-coordinate deletions and retain the subset with largest $J$.  Exact ties choose the first feature in the current candidate order because replacement occurs only for a strict improvement.  We continue to three coordinates, including steps where $J$ declines, then select the global maximum along the nested path.  The maximum is $J=0.720565$ at nine coordinates.  Table~\ref{tab:selection_path} gives every deletion; the selected nine are exactly those defined above.  After this search, the feature set is frozen and the favorable development signs are discarded.  Every ID-only result uses the rule that larger radius receives a higher anomaly score.

\begin{table}[H]
\caption{\textbf{Complete greedy backward-elimination trace.}  ``Deleted'' is the coordinate removed to reach that row; $J$ is Equation~\ref{eq:selection_objective}.  The search continued through three features and selected the global path maximum.}
\label{tab:selection_path}
\centering
\scriptsize
\begin{tabular}{rll}
\toprule
\textbf{Features} & \textbf{Deleted} & \textbf{$J$} \\
\midrule
\csname @@input\endcsname tables/tab_feature_selection_path.tex 
\bottomrule
\end{tabular}
\end{table}

Both $\sigma_{t,F}$ and $J(F)$ use the same development labels, so the selected feature set is exploratory.  The nine features are not independently hypothesized tests and require evaluation on the frozen confirmation data.

\subsection{Category-F ablation with fixed score direction}

\begin{table}[H]
\caption{\textbf{Fixed-direction Category-F ablation.}  All rows use larger Mahalanobis radius as OOD.  ADG-7 removes the two selected component features; ADG-5 is a separately identified five-feature A/D/G subset.}
\label{tab:category_f_ablation}
\centering
\small
\begin{tabular}{lccccc}
\toprule
\textbf{Subset} & \textbf{Toxic} & \textbf{Jailbreak} & \textbf{Hate} & \textbf{Spam} & \textbf{Mean} \\
\midrule
\csname @@input\endcsname tables/tab_category_f_ablation.tex 
\bottomrule
\end{tabular}
\end{table}

Removing F from the selected representation lowers mean AUROC from $0.628$ to $0.579$, primarily because Jailbreak falls from $0.850$ to $0.627$; Spam is unchanged and the two reversed tasks remain reversed.  The compact ADG-5 subset reaches $0.615$.  These fixed-direction AUROCs are lower than the selection objective in Table~\ref{tab:selection_path} because $J$ chooses a favorable sign per task.  Category F is empirically useful for this jailbreak shift; the observed gain does not establish universal transfer or length independence.

\begin{table}[H]
\caption{\textbf{Category-F gain under progressively narrower length bins.}  Values average fixed-direction AUROC across five matching seeds and four tasks; $\Delta_J$ is the ADFG-9 minus ADG-7 gain on Jailbreak.}
\label{tab:category_f_bins}
\centering
\small
\begin{tabular}{rcccc}
\toprule
$B$ & \textbf{ADFG-9} & \textbf{ADG-7} & $\Delta$ & $\Delta_J$ \\
\midrule
\csname @@input\endcsname tables/tab_category_f_bins.tex 
\bottomrule
\end{tabular}
\end{table}

The overall F gain remains $0.040$--$0.049$ and the Jailbreak gain remains $0.181$--$0.223$ as $B$ increases from 5 to 20.  The effect therefore survives progressively narrower length matching, although residual within-bin length effects cannot be ruled out.

\paragraph{Category-F length control.}
In a synthetic truncation sweep at $T\in\{32,64,128,256,512\}$ with source text held fixed, the Category-F attention-ratio mean and standard deviation have $|r|=0.63$ and $0.71$ with length.  They remain range-bounded, but they are not length-independent.  Length matching controls the evaluated length distribution; it does not turn these features into invariant quantities.

\section{Input-Domain Comparison Methods}
\label{app:adaptations}

\subsection{CED-1C}

For target embedding $e(x)$, CED first selects the three nearest ID oracle examples.  For each oracle, it ranks ID auxiliary candidates by cosine opposition between $e(x)-e(a)$ and $e(x)-e(o)$, keeps the 100 least similar, and selects the five farthest.  Text pairs are concatenated as ``target. auxiliary,'' matching the released implementation.  Let $D_c(z)$ be squared Mahalanobis distance to the selected class reference.  Our OOD-positive adaptation is
\begin{equation}
 s_{\mathrm{CED}}(x)=D_c(x)-0.7\,\E_o\left[D_c(o)\right]
 +0.2\,\E_{o,a}\left[D_c(x\oplus a)-D_c(o\oplus a)\right].
\end{equation}
The source implementation uses negative squared distance as ID confidence; this equation is its OOD-oriented negation.  We change the source setting to one ID class, replace RoBERTa/OpenAI embeddings with Qwen, and use final-token decoder pooling.  The source paper and code are available at \url{https://aclanthology.org/2024.findings-emnlp.874/} and \url{https://github.com/MLAI-Yonsei/CED}.

\subsection{\texorpdfstring{Input-D$^2$H}{Input-D2H}}

At each layer, let $\bar h_l=T^{-1}\sum_{t=1}^{T}h_l^{(t)}$ be the mean input-token state. The dispersion component is
\begin{equation}
 D_l=\frac{1}{T}\sum_{t=1}^{T}\|h_l^{(t)}-\bar h_l\|_2.
\end{equation}
For drift, attention from the final input token is averaged across heads, the top 50\% token positions are selected, and their mean representation $g_l$ is formed.  Drift is $\operatorname{mean}_l\|g_{l+1}-g_l\|_2$.  We standardize the two components with ID-train means and standard deviations and average them equally, as the source gives equal weights but does not specify a portable normalization estimator.  The source interprets larger components as faithful-response confidence, so our fixed OOD score is its negative.  Applying generated-response equations to input tokens and choosing ID-only normalization make this an adaptation.

\subsection{Input-TV}

For each input, let $y_l=T^{-1}\sum_t h_l^{(t)}$.  ID training inputs provide coordinatewise $\mu_l$ and $\sigma_l^2$.  Following the paper's $k=0$ equation and the released code's diagonal distance mapping, we compute
\begin{equation}
 d_l=\left\|\frac{(y_l-\mu_l)^2}{\max(\sigma_l^2,10^{-10})}\right\|_2,
 \qquad
 s_{\mathrm{TV}}=\frac{1}{L-1}\sum_{l=1}^{L-1}|d_{l+1}-d_l|.
\end{equation}
The square, variance floor and division apply coordinatewise. The norm of squared deviations follows the released mapping; it is not the usual diagonal Mahalanobis distance. Published TV-Score averages hidden states over generated decoding positions from an ID-fine-tuned model and also studies higher-order smoothing. Our score uses input positions, a pretrained model, and $k=0$ fixed before evaluation.

\section{Data and Reproducibility}
\label{app:data}

\paragraph{Reference scaling.}
Trajectory radius uses \texttt{StandardScaler} fitted on ID references only: subtract each coordinate's mean and divide by its population standard deviation, using scale one for constant coordinates. The sample covariance and ridge in Equation~\ref{eq:trajectory_radius} apply after this transformation.

\paragraph{Confirmation and matching.}
Confirmation starts after the first 200 test examples per class, with equal quotas set by the smaller remainder. Filtering can leave unequal counts. The six-token cutoff is an extraction criterion, not a requirement of the layer-third definitions. Main-analysis bins use pooled linear percentiles and half-open intervals, with one added to the final edge. Equal lengths stay together; bins missing either class are skipped. Each seed initializes one generator used across bins. Each backbone uses its own truncated token counts, so matched rows can differ. Appendix~\ref{app:ced_coverage} gives the separate comparison protocols.

We divide pooled token counts into $B=\min(10,\max(1,\lfloor\min(n_{\mathrm{ID}},n_{\mathrm{OOD}})/10\rfloor))$ equal-frequency bins. Development uses ten bins; confirmation uses nine for ToxicChat, five for Jailbreak, and ten otherwise. Each matched subset combines the retained bins. We compute AUROC on the whole subset, then average across seeds $\{42,123,456,789,1337\}$. Trajectory-radius development retains 246, 172, 364, and 118 total examples for ToxicChat, Jailbreak, HateSpeech and Spam, respectively.

\begin{table}[H]
\caption{\textbf{Checksummed disjoint data partitions.}  Feature selection uses 200 development examples per class; the remaining balanced source partition is reserved for confirmation.}
\label{tab:data}
\centering
\scriptsize
\resizebox{\columnwidth}{!}{%
\begin{tabular}{llllccccc}
\toprule
\textbf{Task} & \textbf{Source} & \textbf{ID} & \textbf{OOD} & \textbf{ID train} & \textbf{ID val} & \textbf{ID test} & \textbf{OOD val} & \textbf{OOD test} \\
\midrule
ToxicChat & lmsys/toxic-chat & non-toxic & toxic & 250 & 50 & 500 & 50 & 300 \\
Jailbreak & jackhhao/jailbreak-classification & benign & jailbreak & 200 & 50 & 250 & 50 & 400 \\
HateSpeech & tweets\_hate\_speech\_detection & normal & hate & 250 & 50 & 500 & 50 & 500 \\
Spam & ucirvine/sms\_spam & ham & spam & 250 & 50 & 500 & 50 & 500 \\
\bottomrule
\end{tabular}%
}
\end{table}

All split files carry SHA256 checksums and are verified before extraction.  An optimized confirmation extractor agrees with the development implementation to maximum absolute error below $2.4\times10^{-8}$ on parity examples.  The comparison-method run took 708.5 seconds on one RTX 3090; confirmation took 166.1 seconds.  The project environment used Python 3.13.7, PyTorch 2.9.1, and Qwen3-0.6B in bfloat16.  The implementation saves resumable feature caches, aggregate JSON, and per-sample NPZ scores.

\paragraph{Source audit and overlap.}
HateSpeech uses the binary Tweets Hate Speech Detection collection.  It does not use Davidson's three-class corpus.  Both jailbreak names use the training split of \texttt{jackhhao/jailbreak-classification}, whose benign and jailbreak examples have different collection origins.  Length control does not remove that source confound.  Hashing all 21 cached trajectory coordinates finds 175 exact ID evaluation-row matches and 194 OOD matches between JailbreakClassification and development, plus 17 cohort ID-evaluation matches to the fitted reference.  These counts measure feature identity.  These caches lack per-row text hashes and pinned dataset revisions, so exact text-level disjointness and the separate crossover resample cannot be certified.  Civil Comments has pinned revisions and text hashes.  Its official-test OOD pool has four duplicate rows, and the intervals do not account for duplicate clusters.

\paragraph{Archive scope.}
The code supplement provides cached numeric inputs for the selected analyses and paired estimates for the intervention, readout and endpoint interval checks.  It excludes raw text, model weights, feature extraction and inputs for other appendix sensitivity analyses; end-to-end reruns require the external datasets and checkpoints described above.

\section{Dispersion Intervention and Robustness Checks}
\label{app:robustness}

\subsection{Fixed-centroid dispersion intervention}
\label{app:dispersion_intervention}

The intervention fixes all six cases, nine features, five matching seeds, 1,000 bootstrap replicates, RNG seed 20260916, and residual scales $\{0.5,0.75,1,1.25,1.5,2\}$.  These choices were set after observing the unmodified scores but before computing intervention outcomes.  For each matched subset, Equation~\ref{eq:dispersion_intervention} preserves the OOD centroid and multiplies precision-weighted population dispersion by $\alpha^2$.  The implementation transforms unstandardized coordinates; affine ID standardization makes this equivalent to the standardized expression.  Tests verify unchanged ID rows/scores, centroid preservation, the dispersion law and baseline parity with the unmodified scores.  The equalized mean-square difference equals the centroid contribution to numerical precision in every case.

\begin{table}[H]
\caption{\textbf{Every prespecified intervention outcome.} The first four rows use confirmation data. $n_m$ is retained per class; $\bar\alpha_{\mathrm{eq}}$ is the mean subset-specific equalization factor. AUROCs are five-seed means. Intervals are pointwise 95\% paired-bootstrap intervals. The final column is the equalized OOD-minus-ID mean squared radius.}
\label{tab:dispersion_intervention}
\centering
\scriptsize
\resizebox{\columnwidth}{!}{%
\begin{tabular}{lrrcccr}
\toprule
Evaluation & $n_m$ & $\bar\alpha_{\mathrm{eq}}$ & Observed & Equalized [95\% CI] & Change [95\% CI] & $\Delta_{\mathrm{rad,eq}}$ \\
\midrule
\csname @@input\endcsname tables/tab_dispersion_intervention.tex 
\bottomrule
\end{tabular}%
}
\end{table}

Each bootstrap replicate resamples evaluation rows within class, rematches lengths using all five seeds, and recomputes centroids, dispersions and equalization factors.  Both scores use the same resampled rows.  Intervals condition on the fixed reference fit and selected features.  They are retrospective and do not control family-wise error.  Figure~\ref{fig:dispersion_intervention}'s baseline intervals use this paired procedure and seed, rather than the bootstrap draws in Table~\ref{tab:score_uncertainty}.  The code supplement includes the paired AUROC estimates; matched indices and point scores can be recomputed from its feature arrays and matching seeds.

\begin{figure}[htbp]
\centering
\includegraphics[width=\columnwidth]{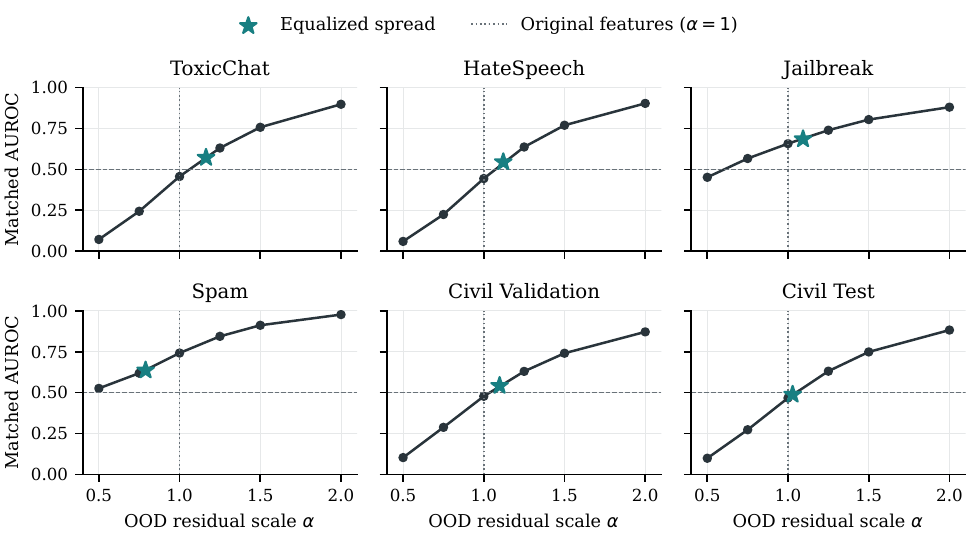}
\caption{\textbf{Full fixed-scale response curves.} No scale is selected to maximize AUROC. The dotted vertical line marks the original features ($\alpha=1$). Each curve averages five matching seeds; stars show mean equalization scale and mean equalized AUROC, with the factor recomputed per subset. These descriptive curves do not carry simultaneous confidence bands.}
\label{fig:dispersion_scale_curves}
\end{figure}

For failure analysis, each OOD example receives its tie-aware rank against the unchanged matched ID scores.  Averaging these ranks recovers AUROC exactly.  The underlying analysis records rank increases, decreases and unchanged cases for each matching seed; the cached numeric features and fixed matching seeds allow these ranks to be recomputed.  Civil test shows why means and rankings must be checked separately: its equalized mean squared-radius difference is positive, its AUROC point estimate is below $0.5$, and its paired-change interval includes zero.  Spam's negative paired change is also retained.

\subsection{Inference sensitivity and typicality control}
\label{app:readiness}

These sensitivity checks were designed after observing the main score results.  We permute labels within fixed pooled length bins (1,999 replicates), preserving each bin's class counts.  We also independently bootstrap reference rows and each evaluation class while refitting standardization and precision (1,000 replicates).  Both use seed 20260915 and the five main-analysis matching seeds.  Permutation within bins assumes that labels are exchangeable only among examples in the same bin.  It does not remove residual within-bin length effects.  The refit intervals still condition on candidate design and selected features, and assume independent rows without accounting for shared source clusters.

\begin{table}[H]
\caption{\textbf{Fixed-score sensitivity checks on cached evaluations.} The first four rows use LOTO features, unlike the confirmation cases in Figure~\ref{fig:dispersion_intervention}. Full radius uses all ID references; CI resamples reference and evaluation rows. $p_B$ is the unadjusted within-bin permutation value. Split radius and density share a half-reference fit; the other half calibrates density.}
\label{tab:readiness}
\centering
\scriptsize
\begin{tabular}{lccccc}
\toprule
Evaluation & Full radius & Refit 95\% CI & $p_B$ & Split radius & Density \\
\midrule
\csname @@input\endcsname tables/tab_readiness_sensitivity.tex 
\bottomrule
\end{tabular}
\end{table}

The four LOTO bin tests remain significant after Holm correction (maximum $p=0.016$).  The Civil test two-stage adjusted value is $0.004$, whereas validation remains inconclusive.  ToxicChat's refit upper endpoint remains close to chance.  These controls support the reported ranking pattern on the tested data.

For an ID-only control inspired by statistic typicality~\citep{morningstar2021density}, we split reference rows with seed 20260915, fit geometry on the first $\lfloor n/2\rfloor$, and fit a one-dimensional Gaussian kernel density estimate (KDE) with Scott's bandwidth to radii of the remaining ID rows.  Negative log density is fixed as the OOD score.  This estimates the density of one statistic and does not implement full DoSE.  It uses LOTO features for the four development tasks and fixed Civil features; its half-reference baseline therefore differs from full-reference confirmation.  No evaluation labels choose its split, bandwidth or direction.  Civil test rises from $0.468$ to $0.502$, but Jailbreak falls from $0.748$ to $0.666$ and Spam from $0.831$ to $0.691$.  These descriptive point differences do not establish paired significance.  All six outcomes and per-row scores are released.

\subsection{Precision-matrix sensitivity}

\begin{table}[H]
\caption{\textbf{Radial metric sensitivity.}  Entries are fixed-direction, five-seed length-matched AUROC.  Features and ID standardization are unchanged; rows vary only the precision matrix.  Euclidean uses identity precision after standardization, and Maha rows add the stated ridge to the ID covariance.}
\label{tab:radial_metric_sensitivity}
\centering
\scriptsize
\resizebox{\columnwidth}{!}{%
\begin{tabular}{lcccccccc}
\toprule
& \multicolumn{4}{c}{\textbf{Development}} & \multicolumn{4}{c}{\textbf{Confirmation}} \\
\cmidrule(lr){2-5}\cmidrule(lr){6-9}
\textbf{Geometry} & \textbf{Toxic} & \textbf{Jail} & \textbf{Hate} & \textbf{Spam} & \textbf{Toxic} & \textbf{Jail} & \textbf{Hate} & \textbf{Spam} \\
\midrule
\csname @@input\endcsname tables/tab_radial_metric_sensitivity.tex 
\bottomrule
\end{tabular}%
}
\end{table}

ToxicChat and HateSpeech each remain below chance in all 12 split--geometry combinations, and Spam remains above chance in all 12.  Development Jailbreak is above chance throughout.  On confirmation it reverses under standardized Euclidean distance and a ridge of $1$, while intermediate ridges remain above chance.  Reversal on ToxicChat and HateSpeech persists beyond the default $10^{-3}$ ridge.  Jailbreak's ranking depends on covariance geometry.

\subsection{All candidates and Civil length balance}
\label{app:additional_robustness}

Table~\ref{tab:all_candidate_robustness} uses all 21 development-designed candidates, without subset elimination or an AUROC-selected transformation.  It keeps the four development AUROC estimates on the same side of chance and gives Civil point estimates below chance.  This tests dependence on the nine-coordinate subset, not on the design of the candidate pool.  The four same-source confirmation caches contain only the selected nine coordinates, so the all-candidate check is unavailable there.  Development and bin-sensitivity results were known before the dispersion diagnostic; these checks do not provide independent confirmation.

\begin{table}[H]
\caption{\textbf{All-candidate sensitivity without subset selection.} Fixed-orientation, five-seed length-matched AUROC. The 21 candidates were designed on the development suite.}
\label{tab:all_candidate_robustness}
\centering
\small
\begin{tabular}{llcc}
\toprule
Stage & Evaluation & Selected 9 & All 21 \\
\midrule
\csname @@input\endcsname tables/tab_all_candidate_robustness.tex 
\bottomrule
\end{tabular}
\end{table}

\begin{table}[H]
\caption{\textbf{Civil Comments across every recorded bin setting.} No bin count is chosen by performance. $n_m$ reports retained examples per class; AUROC uses the unchanged nine-feature score.}
\label{tab:civil_length_robustness}
\centering
\small
\begin{tabular}{rcrcr}
\toprule
Bins & Validation AUROC & $n_m$ & Test AUROC & $n_m$ \\
\midrule
\csname @@input\endcsname tables/tab_civil_length_robustness.tex 
\bottomrule
\end{tabular}
\end{table}

At the canonical ten bins, the five-seed mean length-only AUROC is $0.4977$ for validation and $0.5013$ for test; the corresponding standardized mean differences (OOD minus ID, pooled within-class standard deviation) are $-0.0109$ and $-0.0068$.  Test retains 872 of 1,000 examples per class.  Small scalar balance diagnostics do not rule out nonlinear or within-bin confounding, but the negative test direction is not isolated to one discretization or a large residual mean-length difference.

\subsection{Repeated data draws}

\begin{table}[H]
\caption{\textbf{Fixed-direction repeated-split evaluation} over 100 Monte Carlo reference/OOD redraws.  Score orientation remains fixed in every repeat.  The 95\% range is the empirical 2.5th--97.5th percentile.}
\label{tab:repeated}
\centering
\small
\begin{tabular}{lccc}
\toprule
\textbf{Task} & \textbf{Mean $\pm$ std} & \textbf{95\% range} & \textbf{Fixed prefix} \\
\midrule
\csname @@input\endcsname tables/tab_repeated_split.tex 
\bottomrule
\end{tabular}
\end{table}

The redraw means reproduce the direction of every development result: ToxicChat and HateSpeech remain below chance, while Jailbreak and Spam remain positive.  Cross-draw standard deviations are 0.025--0.041, much larger than variation from length-matching seeds alone.

\subsection{Comparable fixed-score uncertainty}

Table~\ref{tab:score_uncertainty} applies the same stratified bootstrap to each fixed score.  These intervals measure evaluation-sample uncertainty, not just variation across length-matching seeds.

\begin{table}[H]
\caption{\textbf{Fixed-orientation AUROC with stratified-bootstrap 95\% intervals.}  ``Radius dev.'' is exploratory development evaluation; ``Radius conf.'' is the held-out confirmation split from the same source.}
\label{tab:score_uncertainty}
\centering
\scriptsize
\resizebox{\columnwidth}{!}{%
\begin{tabular}{lccccc}
\toprule
\textbf{Task} & \textbf{CED-1C} & \textbf{Input-D$^2$H} & \textbf{Input-TV} & \textbf{Radius dev.} & \textbf{Radius conf.} \\
\midrule
\csname @@input\endcsname tables/tab_score_uncertainty.tex 
\bottomrule
\end{tabular}%
}
\end{table}

\subsection{Bin-count sensitivity and overlap support}

Length matching changes the evaluated population. Table~\ref{tab:raw_matched} reports original and matched AUROC; the checks below vary the bins while preserving score orientation. Matched performance cannot be assumed to hold at all lengths.

\begin{table}[H]
\caption{\textbf{Development bin sensitivity.} Cells give five-seed mean AUROC (retained count per class). Counts do not vary across seeds. Features and score direction are fixed.}
\label{tab:bins}
\centering
\small
\resizebox{\columnwidth}{!}{%
\begin{tabular}{lcccccc}
\toprule
\textbf{Task} & $B=5$ & $B=8$ & $B=10$ & $B=15$ & $B=20$ & \textbf{Spread} \\
\midrule
\csname @@input\endcsname tables/tab_bin_sensitivity.tex 
\bottomrule
\end{tabular}%
}
\end{table}

The maximum movement across $B\in\{5,8,10,15,20\}$ is 0.021, so bin count does not explain the reversals.  At $B=10$, matching changes the OOD/ID mean-length ratio from 2.29 to 0.96 (ToxicChat), 3.27 to 1.03 (Jailbreak), 1.00 to 0.95 (HateSpeech), and 2.34 to 0.98 (Spam).  Effective sample retention is 65\%, 43\%, 91\%, and 30\%, respectively.  The low Spam retention means its matched result applies to a limited overlap region.

\section{Endpoint Comparisons and Additional Task Coverage}
\label{app:ced_coverage}
\label{sec:representation_scores}

\paragraph{Paired confirmation comparison.}
The embedding (endpoint) $k$-NN score uses final-layer token means and cosine distance to ten ID neighbours.  On shared confirmation inputs, it leads on Spam ($0.891$ versus $0.742$ for trajectory radius).  On ToxicChat, trajectory radius is higher ($0.457$ versus $0.302$), but both scores remain below chance.  Paired intervals support these opposite differences; Jailbreak and HateSpeech remain inconclusive.  This retrospective comparison changes both representation and scoring rule.

\begin{table}[H]
\caption{\textbf{Fixed scores on shared confirmation inputs and ID references.} Qwen3-0.6B AUROCs average five token-matching seeds; $n_m$ is retained per class. $\Delta$ is trajectory minus endpoint AUROC. Intervals use 1,000 paired row-bootstrap draws with rematching. They are retrospective, pointwise and conditional on fixed references.}
\label{tab:paired_confirmation}
\centering
\small
\begin{tabular}{lrcccc}
\toprule
\textbf{Task} & $n_m$ & \textbf{Endpoint} & \textbf{Trajectory} & $\Delta$ & \textbf{95\% interval} \\
\midrule
\csname @@input\endcsname tables/tab_aligned_confirmation.tex 
\bottomrule
\end{tabular}
\end{table}

Table~\ref{tab:paired_confirmation} uses the four-task confirmation split, not an independent source replication.  Endpoint embeddings and trajectory features come from the same forward passes with identical ID-reference members: 236, 200, 248 and 244 for ToxicChat, Jailbreak, HateSpeech and Spam.  Endpoint embeddings average final-layer states over all input positions; trajectory summaries omit the final position, as elsewhere.  Both scores treat larger values as OOD.  In each of 1,000 paired bootstrap draws, rows are resampled within class and token-length matching is repeated with the five standard seeds.  We subtract endpoint AUROC from trajectory AUROC within each draw, then take the 2.5th and 97.5th percentiles.  These pointwise intervals condition on the reference fit and feature/source design; matching seeds are not independent replications.  No score calibration or combination enters this comparison.

\paragraph{Development comparison.}
Table~\ref{tab:crossover} compares five fixed scores across six development task contrasts.  It shares evaluation rows across extractors, using ten token-length bins and seed 42 reset within each bin. Separately filtered reference pools are truncated to equal sizes, but their members can differ. The available aggregates cannot recover paired intervals for these estimates; Table~\ref{tab:paired_confirmation}'s intervals do not apply to them. AGNews compares World (ID) with Sports (OOD); JailbreakClass. resamples the Jailbreak source, not an independent source.

\begin{table}[H]
\caption{\textbf{Development comparison on shared matched evaluation rows.} Descriptive Qwen3-0.6B AUROCs with fixed score directions, ten token-length bins and seed 42. ID-reference members can differ. Bold marks the larger endpoint/trajectory estimate, not a significance test.}
\label{tab:crossover}
\centering
\small
\begin{tabular}{lccccc}
\toprule
\textbf{Task} & \textbf{Energy} & \textbf{MaxLogit} & \textbf{Perplexity} & \textbf{Embedding $k$-NN} & \textbf{Trajectory radius} \\
\midrule
\csname @@input\endcsname tables/tab_unified_crossover.tex 
\bottomrule
\end{tabular}
\end{table}

\paragraph{Logit and likelihood scores.}
Energy averages $-\log\sum_v\exp\ell_t(v)$ and MaxLogit averages $-\max_v\ell_t(v)$ over all input positions $t$, including the last; $\ell_t(v)$ is the vocabulary logit at temperature one. The Perplexity column stores mean next-token negative log likelihood, $-(T-1)^{-1}\sum_{t=1}^{T-1}\log p(x_{t+1}\mid x_{1:t})$, for input tokens $x_{1:T}$. Exponentiation gives perplexity with unchanged AUROC. All three scores use larger values for OOD with no evaluation-label sign selection.

\paragraph{Additional classification datasets.}
Table~\ref{tab:cedcoverage} covers six CED datasets unused in feature design. Each detector fits one pooled ID distribution, which can contain several source classes:
\begin{itemize}[leftmargin=*,itemsep=0pt,topsep=2pt]
  \item \textbf{AGNews:} ID = World (0); OOD = Sports (1).
  \item \textbf{Banking77:} ID = labels 0--49; OOD = labels 50--76.
  \item \textbf{SNIPS:} ID = AddToPlaylist, BookRestaurant, GetWeather, PlayMusic, RateBook (first five alphabetically); OOD = SearchCreativeWork, SearchScreeningEvent.
  \item \textbf{News-Categories:} ID = Politics, Wellness, Entertainment, Travel, Style \& Beauty (five most frequent); OOD = all other categories.
  \item \textbf{CLINC150:} ID = archived CLINC-OOS in-scope intents; OOD = out-of-scope class.
  \item \textbf{ROSTD:} ID = ROSTD+ training rows marked ID; OOD = test rows not marked ID.
\end{itemize}
The first four sample source training splits. All six match whitespace-separated word counts, not tokens; the caches lack row identities needed to recover token lengths.

\begin{table}[H]
\caption{\textbf{Endpoint--trajectory comparison.} Qwen3-0.6B uses up to 1,000 ID/500 OOD inputs, reserving half the valid ID inputs as references. $n_m$ is the matched count per class, excluding references. AUROC uses fixed directions, ten word-length bins and seed 42.}
\label{tab:cedcoverage}
\centering
\small
\begin{tabular}{lrcc}
\toprule
\textbf{Dataset} & $n_m$ & \textbf{$k$-NN endpoint} & \textbf{Trajectory radius} \\
\midrule
\csname @@input\endcsname tables/tab_ced_coverage.tex 
\bottomrule
\end{tabular}
\end{table}

Endpoint $k$-NN has the larger AUROC on all six datasets.  These results concern the tested intent/topic shifts and do not establish that vocabulary statistics alone determine performance.  Both scores perform poorly on Banking77, which remains in the table.

\end{document}